\documentclass[letterpaper]{article}
\usepackage{aaai19}
\usepackage{times}
\usepackage{helvet}
\usepackage{courier}
\usepackage[hyphens]{url} 
\usepackage{graphicx}
\usepackage{tikz}
\usepackage{graphicx}
\usepackage[]{algorithm2e}
\usepackage{amsmath}
\usepackage{amsthm}
\usepackage{mathtools}
\usepackage{caption}

\title{Allocation of Multi-Robot Tasks with Task Variants}
\author{Zakk Giacometti and Yu Zhang}

\newtheorem{coro}{Corollary}
\newtheorem{theorem}{Theorem}

\begin{document}
\maketitle

\begin{abstract}
Task allocation has been a well studied problem.
In most prior problem formulations,
it is assumed that each task 
is associated with a unique set of resource requirements.
In the scope of multi-robot task allocation problem, 
these requirements can be satisfied by
a coalition of robots. 
In this paper, 
we introduce a more general formulation of multi-robot task allocation problem that allows more than one option for specifying the set of task requirements--satisfying any one of the options will satisfy the task. 
We referred to this new problem as the multi-robot task allocation problem with task variants. 
First, we theoretically show that this extension fortunately does not impact the complexity class, which is still NP-complete.
For solution methods,
we adapt two previous greedy methods for the task allocation problem without task variants to solve this new problem and analyze their effectiveness. 
In particular, we ``flatten'' the new problem to the problem without task variants, modify the previous methods to solve the flattened problem, 
and prove that the bounds still hold.
Finally, we thoroughly evaluate these two methods 
along with a random baseline to demonstrate their efficacy for the new problem. 
\end{abstract}

\section{Introduction}
Given a set of tasks and a set of resources, 
the problem of task allocation 
is to determine the allocation of resources  (e.g., robots) to tasks so as to 
maximize the overall utility achieved. 
The task allocation problem has been well studied in the robotics community. \cite{gerkey_taxonomy_2004} classifies the problem according to three categories: Single-Task/Multi-Task (ST/MT) robot, Single-Robot/Multi-Robot (SR/MR) task, and Instantaneous/Time-Extended Assignment (IA/TA).
\cite{korash_taxonomy_2013} extends this taxonomy by considering interrelated utilities and constraints among the tasks.
In this paper, we focus on the allocation of multi-robot tasks with
single-task robots and instantaneous assignments. 
This problem has many applications in the real-world,
such as for urban search and rescue,
automated manufacturing and warehousing, etc. 
In this formulation, each task has a set of pre-specified resource requirements and 
each robot is associated with a set of resources (a.k.a. capabilities). 
A task can be satisfied if the set of robots
assigned to it satisfy the resource requirements. 


One assumption made in prior problem formulations is that 
each task is associated with a unique set of resource requirements.
This however may not always be the case for real-world applications. 
Consider a monitoring task for an open area. 
It may be achieved by multiple mobile robots with cameras
or a single UAV. 
The resources required for each way of achieving the task
are very different. 
In this paper, we set out to address the multi-robot
task allocation problem with task variants,
which represent different ways to achieve a task. 
We first theoretically prove that this extension
does not impact the complexity class,
which is still NP-complete. 
To provide a solution, 
we adapt two previous greedy methods that are introduced
for the multi-robot task allocation problem. 
We show that it is not difficult to compile the 
new problem into a ``flattened'' problem,
for which the previous method would apply. 
With slight modifications to the proofs, we show that the solution bounds carry over to the new problem. 
Finally, we thoroughly evaluate these two methods along with a random baseline to demonstrate their efficacy for the new problem.



\section{Related Work}

The multi-robot task allocation problem is known to be NP-complete \cite{gerkey_taxonomy_2004}, and is closely related to the coalition formation problem~\cite{SANDHOLM1999209} in the multi-agent community. 
In fact,
\cite{shehory_methods_1998} first looked at the task allocation problem via coalition formation 
and provided a greedy method based on
the set covering problem. 
\cite{service_coalition_2011} studied the task allocation problem that maximizes utility rather than minimizing cost, and showed that this seemingly innocuous change resulted in very different solution bounds. 
A greedy heuristic was provided to solve this problem. \cite{zhang_considering_2013} further analyzed  
this problem and proposed a new heuristic that incorporates the influence of resource requirements between tasks when making assignments. 
Our work adapted the heuristics in these earlier works to solve the new problem with task variants. 

Our work falls in line with many prior approaches that aimed at extending the applicability of the task allocation problem. 
\cite{vig_coalition_2006,vig_coalition_2007} adapted prior task allocation methods to work in multi-robot systems with additional constraints and preferences (e.g., balanced workload) that are present in physical robotic systems.
\cite{699077} extended the problem to work with decentralized task allocation. 
\cite{6224910} considered dynamic and environmental influences for allocation in distributed robot systems. 
\cite{LIEMHETCHARAT201441} adapted task allocation to accommodate synergies between tasks, and \cite{5980500} studied the problem with precedence constraints between the tasks. 
The Complex Dependencies category in \cite{korash_taxonomy_2013} (e.g. CD[ST-MR-IA]) is also of interest. It describes problems in which a set of subtasks (or task decompositions) must be chosen in addition to the problem of choosing optimal assignments. The optimal task decompositions are not known prior to assigning robots.
A recent work~\cite{cano_solving_2018} that studied the task allocation problem with task variants, applying to the domain of process scheduling. 
However, the problem studied was  single-robot tasks with multi-task robots (i.e., MT-SR)
while we are addressing the multi-robot task allocation problem (i.e., ST-MR). 

On the aspect of task variants,
the information invariant theory~\cite{donald1995information} discussed different ways that a task may be achieved
by different sensori-computational systems,
which are considered equivalent for achieving the task. 
\cite{1570327,6381528} applied this idea
to the problem of coalition formation,
resulting in greater flexibility in dynamic environments compared
to traditional approaches. 
The task variants in our work can be 
considered as static ways of capturing  information invariant for tasks. 


\section{Problem Formulation}
Following prior work, we formulate our variation of the ST-MR-IA problem below, only redefining tasks as sets of task configurations (i.e., task variants). A multi-robot task allocation problem with task variants is a tuple ($R$, $C$, $T$, \textbf{W}, \textbf{V}, $Cost$, $U$):
\begin{itemize}
\item A set of robots $R = \{r_1, r_2, ...\}$. Each robot $r_i$ is associated with a vector $B_i$ of H real non-negative capabilities, in which H is assumed to be a constant that specifies the maximum number of capabilities for a domain.
\item A set of coalitions, $C = \{c_1, c_2, ...\}$. Each coalition $c_j$ satisfies $c_j \subseteq R$.
\item A set of tasks  to be assigned $T = \{t_1, t_2, ...\}$. Each task $t_k$ is associated with a set of task configurations $T_k = \{\tau_{k,1}, \tau_{k,2} ...\}$, where each task configuration $\tau_{k, l}$ requires a vector $P_{k, l}$ of $H$ real non-negative capabilities for achieving task $t_k$ using configuration $\tau_{k, l}$.
\item A vector \textbf{W} of real non-negative costs for capabilities: the use of the capability indexed by $h$ incurs \textbf{W}$[h]$ cost per unit.
\item A vector \textbf{V} of real positive rewards for tasks: achieving task $t_k$ with any of its configurations receives \textbf{V}$[k]$ reward.
\item A function Cost: $C \times \tau \rightarrow \mathcal{R}^0$ that computes real non-negative communication and coordination costs for an assignment based on the coalition and task configuration pair,
where $\tau$ is used above to denote the union of task configurations for all the tasks.
\item A utility function $U$ for assignments, defined as:

\noindent
\begin{eqnarray}
U_s(m_{jk,l}) = \textbf{V}[k] - \sum_h P_{k,l}[h]\textbf{W}[h] - Cost(c_j, \tau_{k,l}) \\
U(m_{jk, l}) =
\begin{cases}
U_s(m_{jk,l}) & \forall h: \sum_{r_i \in c_j} B_i[h] \geq P_{k,l}[h] \\
0 & \text{otherwise}
\end{cases}
\end{eqnarray}

in which $m_{jk,l}$ represents an assignment of a coalition $c_j$ to a task configuration $\tau_{k,l}$.
\end{itemize}

The problem is then to search for a set of assignments $S$ that maximizes:
\begin{equation}
    \underset{m_{jk,l} \in S}{\sum}U(m_{jk,l})
\end{equation}
subject to the constraints
that no assignments must have overlapping robots
and any task must have at most one task configuration assigned in the solution. 
Next, we analyze the complexity of this new formulation

\begin{theorem}
The decision problem of whether there exists an assignment of no less than a given utility value for the multi-robot task allocation problem with task variants is NP-complete. 
\end{theorem}

The proof is straightforward as verifying the solution of this problem would only take polynomial time,
so the problem is in NP. 
Furthermore, since the task allocation problem without task variants is clearly a special case of this new formulation, which is NP-complete, this new problem must also be NP-complete.



\section{Solution Methods}

Since the new problem is NP-complete, instead of looking for exact solutions, 
we propose to study approximate solutions. 

\subsubsection{Random Task Configuration:}
The first thought is to randomly pick from the set of task configurations for each task,
which essentially turns the new problem into a task allocation problem without variants. 
We can then apply any of the state-of-the art task allocation methods. 
This also becomes our baseline approach for comparison.
This method clearly would perform poorly in situations
where a very bad task configuration has been chosen, e.g., 
it renders all the remaining tasks achievable.

\subsubsection{Flattening Formulation:}

A better idea is to try to convert the new problem into a problem without task variants such that prior task allocation solutions can be applied. 
An obvious solution we consider here is a flattening approach that
treats every task configuration as an independent task.
This ``flattens'' the extra dimension of task variants and allows us to consider all possible task configurations at once while allowing prior methods to be directly applied.
A remaining problem, of course, is 
that this formulation can potentially lead to invalid solutions,
since the same task may be assigned multiple times as different task configurations.
This seems to imply that we cannot consider different task configurations at the same time. 
The dilemma here, hence, is
to incorporate the influences among the task variants when making assignments
while preserving the validity of solutions. 
We will show soon that this in fact is not difficult at all. 

\begin{figure}[!h]
\centering
\includegraphics[width=0.6\columnwidth]{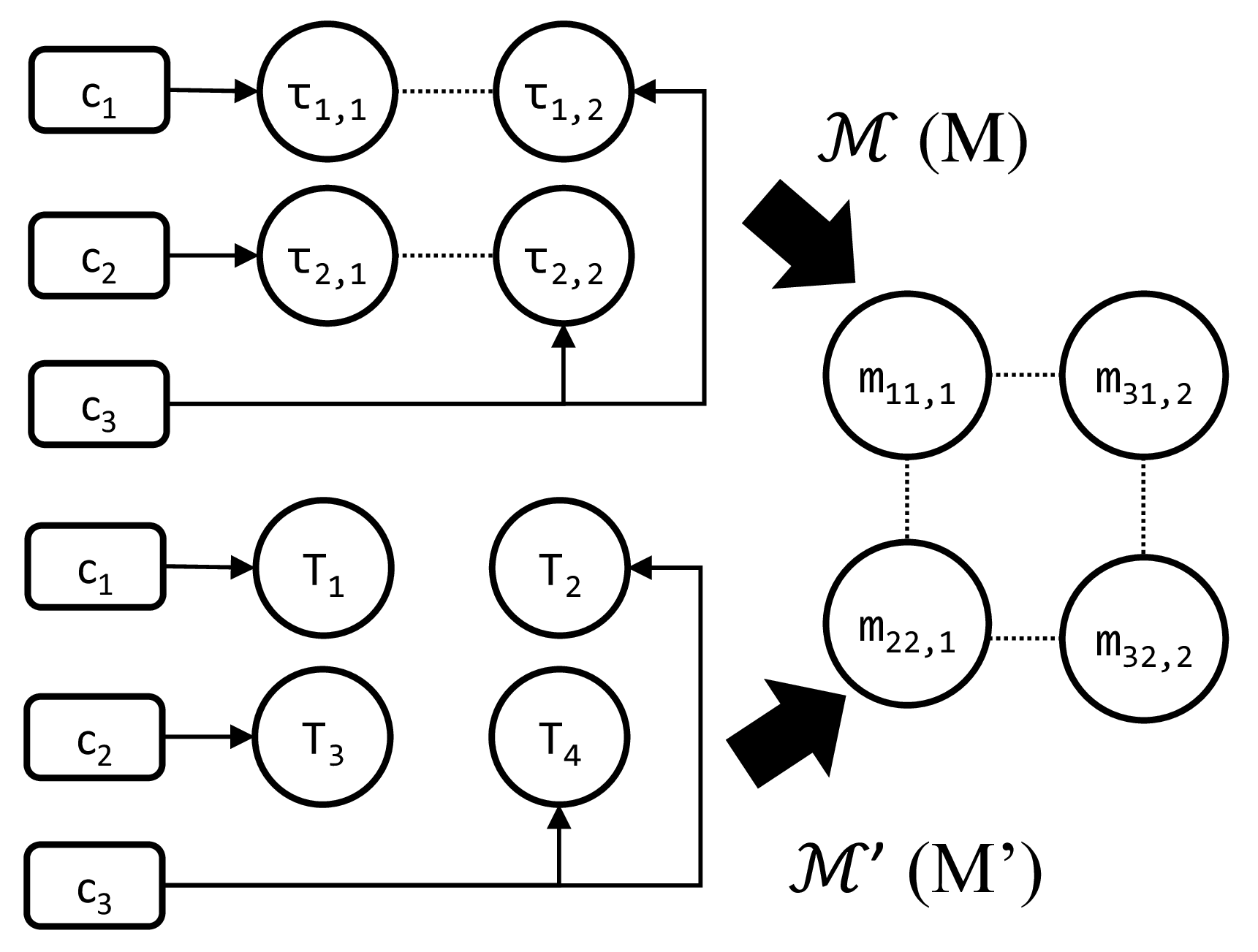}
\caption{An illustration of how our methods maintain valid solutions after flattening. Circles on the left hand side represent task configurations and squares represent coalitions. Circles on the right represent assignments. Arrows indicate a feasible assignment and
dotted lines represent conflicts, e.g., $c_1$ and $c_2$ above conflict (due to overlapping coalitions). The top on the left is a problem without flattening while the bottom with flattening. 
To maintain valid solutions, 
the intuition is to maintain the conflicts in the new problem formulation by updating the definition of $\mathcal{M}$. 
}
\label{fig1}
\end{figure}

\subsection{Maximum Utility with Flattening (FlatMaxUtil)}
First, we introduce a natural greedy heuristic that selects an assignment that maximizes the utility among those remaining at every greedy step,
similar to those in~\cite{service_coalition_2011}.
Since we consider all task variants as separate tasks,
we need to ensure the validity of our solution. A simple way to achieve this is to eliminate assignments to all variants of an assigned task at each greedy iteration.
At each step, the following metric is to be maximized by the chosen assignment:

\begin{equation}
m^\lambda = \underset{m_{xy} \in \mathcal{M'}(\lambda)}{\max} U(m_{xy})
\label{eqfmu}
\end{equation}
where $m_{xy}$ refers to the assignment of coalition $x$ to task $y$. Note that given the flattened formulation, we no longer need to consider the task configuration. 
A special note on the definition of $\mathcal{M}'(\lambda)$ above,
where $\lambda$ refers to the greedy step,
$\mathcal{M}'(\lambda)$ represents the remaining valid assignments to be considered,
and $m^\lambda$ refers to the assignment chosen at the greedy step. 
In~\cite{service_coalition_2011},
there is a definition of $\mathcal{M}(\lambda)$,
where assignments that have coalitions overlapping with the chosen assignment $m^\lambda$ or for the same task will be removed for the next iteration. 
In our formulation, in order to maintain the validity
of the solution, we change $\mathcal{M}(\lambda)$ to $\mathcal{M}'(\lambda)$, which additionally removes assignments that represent different task configurations for the chosen task. 
In this way, we have preserved all the task configurations to be considered at any greedy step
while ensuring that no invalid solution will be produced. 
Fig. \ref{fig1} provides an illustration of this intuition. 
\begin{theorem}
Applying FlatMaxUtil to the ST-MR-IA problem \textbf{with task variants} without restricting the maximum coalition size yields a worst case ratio $\theta = |R + T|$, while restricting the maximum coalition size to be k yields a worst case ratio of $\theta = k+2$.
\end{theorem}

\begin{proof}
Given a task allocation problem with task variants, first, for each task $t_k$, we change the problem by adding a robot $r^k$ that is shared among all the assignments for all the task configurations for $t_k$.
Furthermore, we modify the problem such that each $r^k$ has only a unique capability that is not used by any task. 
This essentially allows at most one of the assignments for a task being made,
which is exactly how we ensure a valid solution. 
As a result, 
the bounds in~\cite{service_coalition_2011} are directly applicable to the flattened problem
after this modification (which are $|R|$ and $k+1$ above, respectively, without task variants). 
Since we add a total of $T$ robots and the maximum coalition size is increased by $1$, 
we have the bounds holds. 
\end{proof}

\subsection{Resource Centric with Flattening (FlatRC)}
The \textit{FlatMaxUtil} method is expected to perform poorly in many scenarios, as it only considers the utility of assignment for each greedy choice and does not consider the influences of assignments on each other.
This effect is first observed in~\cite{zhang_considering_2013}.

\subsubsection{Motivating Example:}
As a motivating example, consider a task allocation problem with 3 tasks, 2 variants per task:
$T = \{t_1, t_2, t_3\}, T_1 = \{\tau_{1,1}, \tau_{1,2}\}, T_2 = \{\tau_{2,1}, \tau_{2,2}\}$ and $T_3 = \{\tau_{3,1}, \tau_{3,2}\}$,
with capability requirements:
$P_{1,1} = (2, 0, 0, 0), P_{1,2} = (1, 1, 0, 1), P_{2,1} = (1, 1, 1, 0), P_{2,2} = (1, 1, 0, 1).$ 
Suppose we only have two robots with the first capability, but sufficient robots with the other three capabilities. Also assume that robots have at most one unit of each capability, all tasks have equal rewards, all capabilities have equal costs, and $Cost$ always returns $0$ for all assignments.
Maximizing solely on utility will cause $\tau_{1,1}$ to be chosen, preventing assignment of either variant of $t_2$, reducing the utility of the final solution.



Similar to the \textit{ResourceCentric} heuristic in~\cite{zhang_considering_2013},
we use a similar heuristic that maximizes the following metric after flattening:


\begin{equation}
  \rho_{xy} = U(m_{xy}) - \underset{m_{jl} \in M_{xy}'(\lambda)}{\sum} \frac{1}{|M'_{jl}(\lambda)|} \cdot U(m_{jl})
\end{equation}
where $M'_{jl}(\lambda)$ represents the set of assignments conflicting with $m_{jl}$ (assignment of $c_j$ to task $t_l$ after flattening), with conflicts defined similarly to how we remove conflicting assignments in $\mathcal{M}'$ in Eq. \eqref{eqfmu}, differing from $\mathcal{M}_{jl}$ in~\cite{zhang_considering_2013}.
It follows that the approximation bounds remain similar to those in~\cite{zhang_considering_2013}:
\begin{coro}
Applying FlatRC to the ST-MR-IA problem \textbf{with task variants} while restricting the maximum coalition size to be k yields a worst case ratio of $\theta = min(2k+4, max_{m_{jl} \in S^*}(|M_{jl}'(1)|))$, in which $S^*$ the optimal solution.
\end{coro}
The proof proceeds nearly identically to that shown for FlaxMaxUtil given the bound in~\cite{zhang_considering_2013} (which is $min(2k+2, max_{m_{jl} \in S^*}(|M_{jl}(1)|))$) .



\subsubsection{Complexity Analysis:}
The algorithm 
for FlatRC follows almost
identically to \textit{ResourceCentric} in \cite{zhang_considering_2013}. 
As we now have multiple configurations per task, the worst case complexity is increased, but only linearly.
For clarity, let $|T_{max}| = \underset{t_k \in T}{max}(|t_k|)$, the size of the largest task  configuration set. Then, 
the complexity is bounded by $O(|T||C||T_{max}||\mathcal{M}|)$, where $\mathcal{M}$ is the set of assignments. Each greedy step is bounded by $O(|\mathcal{M}|^2|)$. As there can be at most $min(|R|, |T||T_{max}|)$ assignments, the overall complexity is bounded by $O(min(|R|, |T||T_{max}|) \cdot |T|^2|T_{max}|^2|C|^2)$.



\subsection{Approximated FlatRC (FlatRCA)}
To improve the computational performance,
we also adapt the \textit{ResourceCentricApprox} heuristic in~\cite{zhang_considering_2013}  to our problem, after flattening. Following a similar reasoning, we wish to reduce the complexity of our algorithm as $|C|$ grows exponentially with $|R|$.
To this end, we compute $\beta_{il}$ = , which measures how much task $t_l$ (note that this is after flattening) depends on robot $r_i$. Then we compute the average expected loss for each \textit{task} $t_l$ due to the assignment of $r_i$, $\varphi_{il}$. Finally, we compute the greedy criteria $\hat\rho_{xy}$ from this value and the utility of each remaining assignment:

\begin{eqnarray}
\beta_{il} = \frac{ |M'_{il}| }{ |M'_{l}| } \\
\varphi_{il} = \overline{ \beta_{il} \cdot U(m_{jl}) }_{m_{jl} \in \mathcal{M}'_{i}(\lambda)} \\
\hat\rho_{xy} = U(m_{xy}) - \underset{r_{i} \in c_{x}}{\sum} \underset{l \neq y}{\sum}  \varphi_{il} 
\end{eqnarray}

\section{Simulation Results}
In this section, we provide simulation results for the task variant problem. We focus mainly on randomly generated allocation scenarios, varying key parameters. In all cases when evaluating performance ratios we compare against the upper bound of the optimal solution as \cite{shehory_methods_1998}: the sum of the feasible assignments with the maximum utility for each task without checking for conflicts. The costs of each capability (i.e. \textbf{W}) are randomly generated from [0.0, 1.0]. Each task or robot has a 50\% chance to need/provide any capability. Capability values, unless specified otherwise, are generated from [0, 8]. The number of capabilities \textbf{H} is fixed at 7. The maximum size of coalitions is fixed at 5 ($k=5$).
Task rewards (i.e. \textbf{V}) are generated randomly from [100, 200]. \textit{Cost} is defined as a linear function of coalition size, $4n$. Measurements are made over 1000 runs.

We make two comparisons: varying the number of robots and tasks. We also compare time when varying robots. Varying the number of task variants showed similar trends to varying robots and is not shown. Our time analysis we only consider the time required to assign coalitions to tasks.
Note that in most of our results, \textit{FlatRC} and \textit{FlatRCA} overlap significantly.
We show a clear improvement in applying \textit{FlatRC} and \textit{FlatRCA} over the simple greedy heuristic for varying numbers of robots, tasks.

\begin{figure}[!ht]
\centering
\includegraphics[width=0.8\columnwidth]{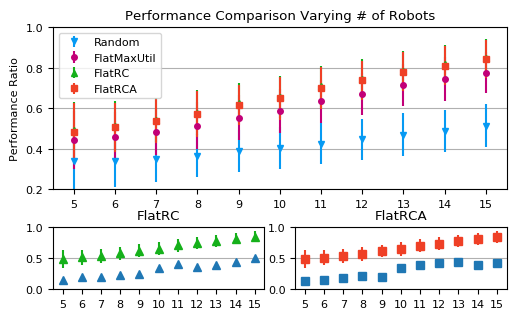}
\caption{Results varying \# of robots available. $|T|$ is fixed at 10 and the maximum \# of configurations per task is 5.}
\label{fig2}
\end{figure}

\begin{figure}[!ht]
\centering
\includegraphics[width=0.7\columnwidth]{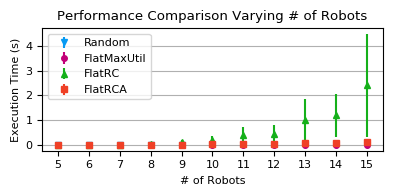}
\caption{Time comparison for Figure~\ref{fig2}.}
\label{fig3}
\end{figure}

\begin{figure}[!ht]
\centering
\includegraphics[width=0.8\columnwidth]{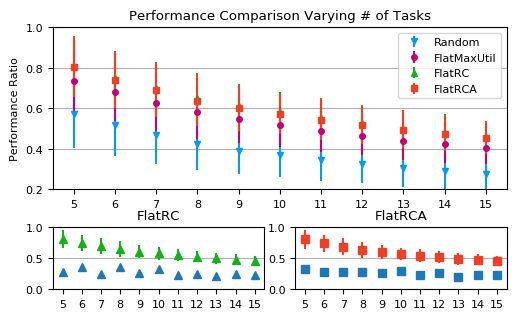}
\caption{Results for varying \# of tasks, $|R|$ is fixed at 8 and the maximum \# of configurations per task is 5.}
\label{fig4}
\end{figure}


\section{Conclusion and Future Work}
First, we introduced a new formulation of the ST-MR-IA problem that allows for more realistic and flexible scenarios of achieving tasks in the form of task configuration variants. A simple but effective method of solving this problem is to ``flatten'' it into a task allocation problem without the variants. With slight modifications, this allows the application of existing greedy heuristics that provide good approximation bounds.
However, this method effectively discards some finer information about the interaction between task variants. It is clear that improved methods that utilize this information may be devised, but the increased complexity of the problem do not make it trivial to do so. In future work, we plan to investigate such a method if it does exist and compare its performance with those discussed in this work.

\bibliography{socs19taskbib}
\bibliographystyle{aaai}

\end{document}